%% file: main.tex
\documentclass{article}

\usepackage{arxiv}

\usepackage{amsmath}
\usepackage{amssymb}
\usepackage{amsthm}
\usepackage[utf8]{inputenc} 
\usepackage[T1]{fontenc}    
\usepackage{hyperref}       
\usepackage{url}            
\usepackage{booktabs}       
\usepackage{amsfonts}       
\usepackage{nicefrac}       
\usepackage{microtype}      
\usepackage{cleveref}       
\usepackage{lipsum}         
\usepackage{graphicx}
\usepackage{natbib}
\usepackage{doi}
\usepackage{bm}
\usepackage{dsfont}
\usepackage{graphicx}
\usepackage{dsfont}
\usepackage{color}
\DeclareMathOperator*{\argmax}{arg\,max}

\usepackage{algorithm}
\usepackage{algorithmic}
\usepackage{float}
\usepackage{url}
\usepackage{tablefootnote}
\usepackage{footmisc}
\usepackage{makecell}

\newtheorem{theorem}{Theorem}
\newtheorem{remark}{Remark}
\newtheorem{proposition}{Proposition}
\newtheorem{lemma}{Lemma}
\newtheorem{corollary}{Corollary}

\title{Pure exploration in multi-armed bandits with low rank structure using oblivious sampler}

\author{ \href{https://orcid.org/0000-0001-8588-4543}{\includegraphics[scale=0.06]{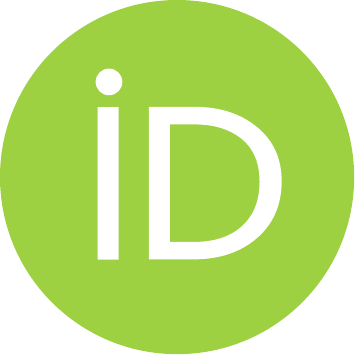}\hspace{1mm}Yaxiong Liu}
  \\
	Graduate School of Information Science and Technology\\
	Hokkaido University\\
	Sapporo, Japan \\
	\texttt{yaxiong.liu@ist.hokudai.ac.jp} \\
	\And
	Atsuyoshi Nakamura\\
	Graduate School of Information Scienece and Technology\\
	Hokkaido University\\
	Sapporo, Japan\\
	\texttt{atsu@ist.hokudai.ac.jp} \\
	 \AND
  \href{https://orcid.org/0000-0002-1536-1269}{\includegraphics[scale=0.06]{orcid.pdf}\hspace{1mm}Kohei Hatano} \\
	Department of Informatics\\
	Kyushu University\\
	Fukuoka, Japan\\
	\texttt{hatano@inf.kyushu-u.ac.jp} \\
 \AND
 \href{https://orcid.org/0000-0001-9542-2553}{\includegraphics[scale=0.06]{orcid.pdf}\hspace{1mm}Eiji Takimoto} \\
	Department of Informatics\\
	Kyushu University\\
	Fukuoka, Japan\\
	\texttt{eiji@inf.kyushu-u.ac.jp} \\
}

\renewcommand{\shorttitle}{\textit{arXiv} Template}

\hypersetup{
pdftitle={A template for the arxiv style},
pdfsubject={q-bio.NC, q-bio.QM},
pdfauthor={David S.~Hippocampus, Elias D.~Striatum},
pdfkeywords={First keyword, Second keyword, More},
}
\begin{document}
\maketitle
\begin{abstract}
In this paper, we consider the low rank structure of the reward sequence of the pure exploration problems. Firstly, we propose the separated setting in pure exploration problem, where the exploration strategy can not receive the feedback of its explorations.  Due to this separation, it requires that the exploration strategy to sample the arms obliviously. By involving the kernel information of the reward vectors, we provide efficient algorithms for both time-varying and fixed cases with regret bound $O(d\sqrt{(\ln N)/n})$. Then, we show the lower bound to the pure exploration in multi-armed bandits  with low rank sequence. There is an $O(\sqrt{\ln N})$ gap between our upper bound and the lower bound.
\end{abstract}


\input{introduction}
\input{preliminaries}
\input{alg2}
\input{conclusion}

\bibliographystyle{unsrtnat}
\bibliography{references} 

\input{appendix}

\end{document}

%% file: introduction.tex
 \section{Introduction}
Pure exploration in multi-armed bandits is first introduced by  \cite{bubeck2009pure} and has been further extensively researched in the academic community and applied in the industry. This problem is based on the stochastic multi-armed bandits and assumes that exploration is in the total process and recommends only the last round $n$. This problem is described as follows: On round $t \in [n]:$ the exploration strategy explores some arm $I_{t} \in [N]$ and then receives the feedback $\bm{l}_{t}(I_{t})$ from the environment with some underlying distribution. Then the recommendation strategy gives a recommendation $J_{t}$. Note that this exploration process ends when the environment determines. The goal of the algorithm is to guarantee that the recommendation is optimised with respect to the \emph{simple regret} maximised.

As we mentioned, the performance of the algorithm is highly-related to the exploration strategy, sampling the arms, and recommendation strategy, selecting the arm. In this paper, we consider the oblivious sampling with a structured reward sequence. 

To the best knowledge of us, all the algorithms for pure exploration assume that the exploration strategy samples the arms adaptively. It implies that the exploration strategy can receive the feedback of its selection $\bm{l}_{t}(I_{t})$ and update the sampling by rounds. However, in our mentioned oblivious sampler, we assume that the feedback can be only received by the recommendation strategy. It means that the exploration strategy is fixed in advance of the learning process, while the recommendation strategy can only give recommendations from the oblivious sample determined by the exploration strategy. 

Meanwhile, we involve the low rank structure of the reward sequence in our work. The low rank structure reward has been first introduced by \cite{hazan2016online} in the online learning framework. Specifically, the $N$ dimensional reward vector $\bm{l}_{t}$ is dependent on $d$ factors ($d$-dimensional seed vector) according to an $N \times d$ kernel matrix under the low rank assumption. 
In the bandit problem, this setting is similar to linear bandits \citep{cesa2012combinatorial} and linear contextual bandits \citep{chu2011contextual}, when the kernel is fixed or stochastic, respectively.
Therefore, there is a straightforward algorithm for pure exploration setting based on linear contextual bandits algorithm when the exploration part can deploy the adaptive sampler and the running round $n$ is given to the exploration strategy in advance \citep{bubeck2009pure}, when the kernel is known to the algorithm. Otherwise, without the prior knowledge of $n$, there is an extra $\ln n$ factor incurred by doubling-trick \citep{besson2018doubling, zimmert2019optimal}. 

Using the oblivious sampler for pure exploration problem can be widely utilised in practice scenarios.
Concretely, in crowdsourcing \citep{ghezzi2018crowdsourcing}, one user will deliver the tasks to the workers with a platform. In this scenario, this user will provide the task $t$ on the platform. Then the platform will select one of the $N$ workers to accomplish this task. Next, this user receives the feedback from the work $I_{t} \in [N]$ and decides which worker is the most appropriate for his tasks according to the historical record. Finally, the process will end, when some worker leaves the platform. In this case, the platform plays the role of the exploration part, and the user plays the recommendation part. For some privacy reasons, the oblivious sampler should be deployed, while the reward is confidential to the platform. Simultaneously, on the low rank setting aspect, the rewards of these $N$ workers are depending on the ability of the workers and the task $t$ itself. Mathematically, we can model the ability of the workers with an $N \times d$ matrix with respect to the seed vector (about the task). We assume that the user can not parameterise the task into a vector, while the platform can parameterise the task but does not know the content due to privacy reasons. More generally, the kernel matrix, representing the ability of the workers,  can be fixed or stochastic. 

Due to our assumption, the UCB or Thompson sampling can not be utilised, since the exploration part can not receive the feedback of its selections. 
A core problem in algorithm designing in this paper is to propose an oblivious sampler for the exploration strategy, which can provide necessary information to the recommendation strategy with respect to the low rank structure. 
Therefore, we propose firstly an efficient algorithm for pure exploration in multi-armed bandits with stochastic kernels.  In this algorithm, the exploration strategy samples the arms uniformly in each round. From the feedback of the exploration arm, the recommendation strategy constructs an unbiased estimator to the seed vector and recommends the estimated best arm with the averaged estimated seed vectors and the kernel matrices. By some restrictions to the kernel matrix and seed vector, our proposed algorithm achieves the regret bound as $O \left( d\sqrt{\frac{\ln N}{n}}\right)$.

Next, we consider the fixed kernel case by releasing the range restriction on the seed vector and the kernel matrix. By involving the Bary-centric spanner \citep{awerbuch2004adaptive}, the exploration strategy focuses on exploring $d$ ``important arms'' from the total $N$ fixed arms during the learning process one by one on each round. Then the recommendation strategy can estimate the seed vector according to the feedback and the corresponding arms. In this designed algorithm, the regret bound is as same as the stochastic case.

Finally, we show a lower bound to the pure exploration problem with low rank structure as $\Omega\left( \frac{d}{\sqrt{n}}    \right)$. There is still a $\sqrt{\ln N}$ gap for the separated setting. 

 \textbf{Main contributions:} In this paper, we focus on the pure exploration problem in multi-armed bandits utilising oblivious sampler.  Concretely, we can summarise as follows:
\begin{itemize}
      \item We consider the low rank reward sequence in pure exploration problem. We propose the algorithms for fixed and stochastic kernels respectively. For the latter case, our algorithm guarantees an upper bound as $O\left( \sqrt{\frac{d^{2} \ln N}{n}}\right)$ by sampling the arms uniformly. For the former case, our proposed algorithm achieves the same regret bound with less restriction on the range of seed vectors and kernel matrix. In this algorithm, we need only obtain the Bary-centric spanner at the beginning of the algorithm and sample the exploration according to these ``important" arms on each round. 
    \item We demonstrate a lower bound $\Omega\left( \sqrt{\frac{d \ln N}{n}}\right)$ for pure exploration problem with low rank structure. It implies that our proposed upper bound is near-optimal.  
\end{itemize}
Here we conclude our contribution in the following chart:

\begin{table}[H]
\caption{Comparison with our work and previous works}
\centering
\begin{tabular}{|l|c|l|}
\hline
  & Fixed kernel& Stochastic kernel\\
\hline
Oblivious sampler &  $O\left(\frac{d\sqrt{\ln N}}{\sqrt{n}}\right)$  $\Omega\left( \sqrt{\frac{d \ln N}{n}}\right)$  & $O\left(\sqrt{\frac{d^{2} \ln N}{n}}\right)$   $\Omega\left( \sqrt{\frac{d \ln N}{n}}\right)$   \\
\hline
Adaptive sampler & \makecell[c]{$\Theta\left( \sqrt{\frac{d \ln N}{n}}\right)$ \\ \cite{bubeck2009pure,bubeck2012towards}}  & \makecell[c]{$O\left( \sqrt{\frac{d^{2}\ln n}{n}}\right)$ \\ \cite{bubeck2009pure,chu2011contextual}}\\
\hline
\end{tabular}
\end{table}


\textbf{Related Work:}
The setting of pure exploration with multi-armed bandits is similar but different from the best arm identification problems \citep{jamieson2014best,madani2004budgeted}. In the best arm identification problem, the budget of exploration $n$ is one of the prior information to the algorithm, while in the pure exploration problem, the stopping signal is given by the environment arbitrarily. 

Low rank sequence in prediction with expert advice can be seen as a full-information setting for multi-armed bandits with low rank structure with the fixed kernel. In a low rank setting, the $N$-dimensional reward vectors are actually from a $d$-dimensional kernel space.
A central work in the low rank setting is to explore the $d$-dimensional seed vector during the learning process. With the prior information about the kernel space, the regret bound to Expert advice with low rank structure is $\Theta(\sqrt{dn})$, which is not related to the dimensional of the reward vector.

Our low rank setting in multi-armed bandits can be seen as linear bandits or linear contextual bandits with respect to the fixed or stochastic kernel respectively. In the former case, 
\cite{bubeck2012regret,dani2007price} provide an upper bound  
$O\left( \sqrt{d n \ln N}\right)$, and lower bound  $\Omega\left( d \sqrt{n}\right)$ to regret, respectively. 
In the latter case, the upper bound is $O\left(  \sqrt{\frac{d^{2} \ln N}{n}}\right)$ \citep{bubeck2012regret}. Again, we need to emphasis, we can not directly utilise the existing algorithm for linear 
 (contextual) bandits, due to the oblivious sampling setting, which is not concerned with traditional pure exploration research \citep{bubeck2009pure}.

%% file: preliminaries.tex
\section{Preliminaries}
First of all, we define some notations used in this paper.
We let $[n]$ denote $\lbrace 1,\cdots, n \rbrace$ for a positive integer $n$.
A vector is denoted by a bold letter and the $i$-th component of a vector $\bm{a}$ is denoted by $\bm{a}(i)$.
A matrix and its $i$-th row is denoted by a capital letter and its corresponding small letter subscripted with $i$, respectively.
For example, the $i$-th row of matrix $U$ is denoted by $\bm{u}_{i}.$

For a positive integer $N$, we define $\Delta(N)$ as
\begin{equation}
\Delta(N)=\left\{  \bm{x}: \sum_{i=1}^{N} \bm{x}(i)=1   \quad \land \quad \bm{x}(i) \geq 0 \quad \forall i\in [N]          \right\}.
\end{equation}

\subsection{Pure exploration in multi-armed bandits}
In this sub-section, let us have a review of the problem setting of ``pure exploration'' in $N$-armed bandits problem proposed by \citet{bubeck2009pure}.

For each round $t=1,2,3,\cdots:$
\begin{enumerate}
\item Algorithm chooses $\varphi_{t} \in \Delta(N)$ and pulls $I_{t} \in [N]$ according to the distribution $\varphi_{t}$.
\item Environment returns reward $\bm{l}_{t}(I_{t})$, where $\bm{l}_{t} \in [-1,1]^{N}$.
\item If the environment sends a stopping signal then 
  \begin{description}
\item[\ \ \ ]  stopping time $n$ is set to the current time $t$ and
\item[\ \ \ ]  the algorithm outputs $J_n\in [N]$ and the game is over,
\end{description}
  otherwise the game continues.
\end{enumerate}
In this paper, we consider the stochastic bandits problem, where we assume that 
the reward sequence $\bm{x}_1,\bm{x}_2,\cdots$ is an i.i.d random variable sequence.
Moreover, we let $\boldsymbol{\mu}$ denote that $\mathbb{E}\left[ \bm{x}_{t}\right]$ for all $t\in \mathbb{N}_{+}.$

The goal of the algorithm is to minimise the simple regret $r_n$, which is defined as follows:
\begin{equation}
r_n=\max_{i \in [N]}\boldsymbol{\mu}(i)-\boldsymbol{\mu}(J_n),
\end{equation}
where $n$ is the stopping time that is decided by the environment.

An algorithm for this problem can be determined by designing a sampling (or allocation) strategy and a recommendation strategy.
A sampling strategy decides how to choose $\varphi_{t}$ at $t=1,2,\cdots$, and a recommendation strategy decides how to choose $J$.
Two types of sampling strategies are considered.
One is \emph{oblivious} and the other is \emph{adaptive}.
An oblivious sampler chooses $\varphi_{t}$ independently of history $I_1,\bm{l}_1(I_1),\cdots,I_{t-1},\bm{l}_{t-1}(I_{t-1})$
while an adaptive sampler chooses $\varphi_{t}$ depending on it.
In this paper, we only consider oblivious samplers as sampling strategies.

\citet{bubeck2009pure} proposed a uniform allocation strategy 
and Empirical Best Arm (EBA) recommendation strategy. 
It means that for each round $t$, the algorithm will select the arm uniformly from $1$ to $N$ for $\varphi_{t}$ (on round $t$ selecting arm $t \mod N$) and outputs
\begin{equation}
J_n=\argmax_{j \in [N]}\frac{1}{|T_{j}(n)|}\sum_{s \in T_{j}(n)}\bm{l}_{s}(j)
\end{equation}
at the stopping time $n$, where $T_{j}(t)=\left\{  s: I_{s}=j,  \forall s \in [t]   \right\}$.

\begin{theorem}[\cite{bubeck2009pure}]
  If we run the algorithm of pure exploration in $N$-armed bandits problem using uniformly allocation and EBA recommendation strategies for $n$ rounds, the simple regret $r_n$ is bounded as follows:
\begin{equation*}
r_n \leq O\left(  \sqrt{\frac{ N \ln N }{ n}} \right).
\end{equation*}
\end{theorem}

\subsection{Low rank structure}

We say that the reward sequence $\bm{x}_1,\bm{x}_2,\cdots$ is $d$-rank reward sequence if and only if
\[
\mathrm{rank}([\bm{x}_{1},\cdots, \bm{x}_{T}])\le d \text{ for all } T\in \mathbb{N}_+,
\]
where $[\bm{x}_{1},\cdots, \bm{x}_{T}]$ is a matrix whose $t$-th column is   $\bm{x}_{t}$.
A $d$-rank reward sequence is said to have the low rank structure when $d \ll N$.
For a $d$-rank reward sequence $\bm{x}_1,\bm{x}_2,\cdots$, there are a kernel matrix $U \in \mathbb{R}^{ N \times d}$
and $d$-dimensional vector sequence $\bm{v}_1,\bm{v}_2,\cdots$ such that $\bm{x}_t=U\bm{v}_t$ for all $t\in \mathbb{N}_+$.
We call $\bm{v}_{t}$ the seed vector of the reward vector $\bm{x}_{t}$.
In our stochastic setting, seed vector sequence $\bm{v}_1,\bm{v}_2,\cdots$ is an i.i.d. random variable sequence 
that satisfies $\bm{v}_t\sim\mathcal{D}$, $U\bm{v}_t\in [-1,1]^N$ and $\mathbb{E}[\bm{v}_t]=\bm{v}$ for some distribution $\mathcal{D}$, some $d$-dimensional vector $\bm{v}$  and  all $t\in \mathbb{N}_+$.

Furthermore, we extend the low-rank structure of the reward sequence from the fixed kernel $U$ to the random kernel sequence $U_1,U_2,\cdots$.
A reward sequence $\bm{x}_1,\bm{x}_2,\cdots$ is said to have the low rank structure for the i.i.d. random kernel sequence $U_1,U_2,\cdots\in \mathbb{R}^{ N \times d}$ with $d \ll N$
if there are seed vector sequence $\bm{v}_1,\bm{v}_2,\cdots$ and a matrix $U \in \mathbb{R}^{ N \times d}$ such that
$\bm{x}_t=U_t\bm{v}_t$ for all $t\in \mathbb{N}_+$ and $\mathbb{E}[U_t]=U$.
We let $\bm{u}_{t,i}$ denote the $i$-th row of $U_{t}$, thus $\mathbb{E}[\bm{u}_{t,i}]=\bm{u}_{i}$ holds. 
The seed vector sequence $\bm{v}_1,\bm{v}_2,\cdots$ is also assumed to be an i.i.d. random variable sequence 
that satisfies $\bm{v}_t\sim\mathcal{D}$, $U_t\bm{v}_t\in [-1,1]^N$ and $\mathbb{E}[\bm{v}_t]=\bm{v}$ for some distribution $\mathcal{D}$, some $d$-dimensional vector $\bm{v}$  and  all $t\in \mathbb{N}_+$.
Further we assume that the kernel matrices and the seed vectors are independent, thus $\mathbb{E}[\bm{x}_{t}]=\mathbb{E}[U_{t}] \cdot \mathbb{E}[\bm{v}_{t}]=U\bm{v}\in [-1,1]^N$.
We also assume that all rows of the kernel matrices $\bm{u}_{t,i}$ and $\bm{u}_i$ belong to a domain $\mathcal{U}$.
Note that the random kernel setting includes the fixed kernel setting.

Finally, we assume that the kernel information $U_{t}$ is given to the algorithm 
in advance on each round $t \in [T]$.

%% file: alg2.tex
\section{Main Result}
In this section, we demonstrate the main results of our paper. In the first sub-section, we give an algorithm for the stochastic kernel with oblivious sampler. In the second sub-section, we design an algorithm for the fixed kernel (low rank reward) case. We need emphasis that we constrain the stochastic kernel and seed vector to obtain the desirable simple regret bound. Finally, we show the lower bound of the pure exploration problem which indicates that there is still a $O(\sqrt{\ln N})$ gap between the current bounds.

\subsection{Low rank reward with stochastic kernel}
In this sub-section, we consider the pure exploration problem in multi-armed bandits with the stochastic kernel. One of the central problems in the separated setting is that the exploration strategy is only depended on the kernel information instead of the feedback of the pulled arm in ordinary setting. Once the exploration strategy receives the kernel information $U_{t}$, the exploration strategy is requested to output an oblivious  distribution to sample the pulled arm $I_{t}$.  A straightforward idea is to explore the arms uniformly, if there is no specified structure upon the reward sequence. With the assistance of the kernel information, in this part, we proposed an efficient sampler by involving the John distribution over all arms \citep{grotschel2012geometric,cesa2012combinatorial}.


Again, our goal is to minimise the simple regret defined as follows:
\begin{equation}\label{Eq:time-varing-simple-regret}
r_{n}=\max_{i \in [N]} \bm{u}_{i} \cdot \bm{v}- \bm{u}_{J_{n}} \cdot \bm{v}.
\end{equation}

In principle, we can recommend the arm more effectively, if the recommendation part can estimate the seed vector $\bm{v}_{t}$. Instead of recommending from $N$-arms, the recommendation strategy performs well, if it is confronted to a lower dimensional space, a $d$-dimensional action space. Firstly, we propose the estimator of the seed vector $\bm{v}_{t}$ as follows:
\begin{equation}
\hat{\bm{v}}_{t}=A_{t}^{-1}\bm{u}_{t,I_{t}}^{\top}\bm{x}_{t}(I_{t}),
\end{equation}
where $A_{t}=\left[  \sum_{i=1}^{N} \bm{p}_{t}(i) \bm{u}_{t,i}^{\top} \bm{u}_{t,i}  \right]=\mathbb{E}_{I_{t} \sim \bm{p}_{t}} \left[ \bm{u}_{t,I_{t}}^{\top} \bm{u}_{t,I_{t}} \right]$ for some exploration distribution $\bm{p}_{t}\in \Delta(N).$ Meanwhile, we see that $\hat{v}_{t}$ is unbiased, since
\begin{equation}
    \mathbb{E}_{I_{t} \sim \bm{p}_{t}}[\bm{\hat{v}}_{t}]=\mathbb{E}_{I_{t} \sim \bm{p}_{t}}[A_{t}^{-1}\bm{u}_{t,I_{t}}^{\top}\bm{x}_{t}(I_{t})]=\mathbb{E}_{I_{t} \sim \bm{p}_{t}}[A_{t}^{-1}\bm{u}_{t,I_{t}}^{\top}\bm{u}_{t,I_{t}}\bm{v}_t]=\bm{v}_{t}.
\end{equation}
Since $\bm{v}_{t}$ is a random variable with respect to some underlying distribution, we have that $\mathbb{E}\left[ \mathbb{E}_{I_{t} \sim \bm{p}_{t}}[\hat{\bm{v}}_{t}]\right]=\bm{v}$.

Thus, the recommendation strategy selects the arm $J_{n}$ by the best estimate arm such that
\begin{equation} \label{algorithm:recommendation}
  J_{n}=\argmax_{i\in [N]} \left(\left(   \frac{1}{n}\sum_{s=1}^{n}U_{s} \right)\left(\frac{1}{n}\sum_{s=1}^{n}\hat{\bm{v}}_{s}\right)\right)(i).
\end{equation}

Then we can demonstrate the algorithm for this low rank loss with stochastic kernel as follows:
\begin{algorithm}
\caption{Algorithm for pure exploration in stochastic kernel bandits}
\label{alg:Algorithm-for-unfixed-kernel}
\begin{algorithmic}
\REQUIRE A uniformly distribution $\bm{p}_{t}=\left( \frac{1}{N}, \cdots, \frac{1}{N},\cdots, \frac{1}{N} \right)$

\FOR {$t=1,2,\cdots$}
\STATE Receive $U_{t}$ the kernel of round $t$ and sample the arm $I_{t}$ according to distribution $\bm{p}_{t}$.
\STATE The recommendation receives the reward $\bm{x}_{t}$ and calculate
\begin{equation}
\hat{\bm{v}}_{t}=A_{t}^{-1}\bm{u}_{t, I_{t}}^{\top}\bm{x}_{t}(I_{t}),
\end{equation}
where $A_{t}=\mathbb{E}_{\bm{p}_{t}}[(\bm{u}_{t,i})^{\top} (\bm{u}_{t,i})]$ .
\IF{the environment sends a stopping signal}
\STATE Set $n$ to $t$.
\STATE Recommend $J_{n}=\max_{j \in [N]} \left(\frac{1}{n}\sum_{s=1}^{n}\bm{u}_{s,j}\right)\left(\frac{1}{n}\sum_{s=1}^{n}\hat{\bm{v}}_{s} \right)$ and stop.
\ENDIF
\ENDFOR
\end{algorithmic}
\end{algorithm}

Note that the simple regret is based on the expectation of the kernel matrix $U=\mathbb{E}[U_{t}],$ although it emerges not during the learning process.  Before we state our main result, let us introduce McDiarmid's inequality firstly.

\begin{theorem}
Let $f: \mathcal{X}_{1} \times \cdots \times \mathcal{X}_{n} \rightarrow \mathbb{R}$ satisfy the bounded difference property with bounds $c_{1},\cdots, c_{n}:$
$$
\forall i \in [n] \sup_{x_{i}^{'} \in \mathcal{X}_{i}}\left| f(x_{1},\cdots, x_{i},\cdots, x_{n})-f(x_{1},\cdots, x_{i}^{'},\cdots, x_{n}) \right| \leq c_{i}.
$$
Consider independent random variable $X_{1},\cdots, X_{n}$ where $X_{i} \in \mathcal{X}_{i}, \forall i.$ Then for all $\epsilon >0$,
\begin{equation}
\mathbb{P}\left[   f(X_{1},\cdots , X_{n})-\mathbb{E}\left[ f(X_{1},\cdots, X_{n})\right] \geq \epsilon         \right] \leq \exp \left(   -\frac{2 \epsilon^{2}}{\sum_{i=1}^{n} c_{i}^{2}}     \right)
\end{equation}
\end{theorem}

The main Theorem is given as follows: 
\begin{theorem}\label{thm:regret-bound-time-varying}
In the pure exploration multi-armed bandits problem with time-varying kernel $U_{t} \in \mathbb{R}^{N \times d}$ according to some underlying distribution. If the exploration process is run by $n$ times with some exploration distribution $\bm{p}_{t}$ and the recommendation strategy is defined as in Equation (\ref{algorithm:recommendation}). The simple regret $r_{n}$ defined in  Equation  (\ref{Eq:time-varing-simple-regret}) is bounded:
\begin{equation}
\mathbb{E}[r_{n}] \leq O \left( \alpha \beta \sqrt{\frac{ \ln N}{n}} \right),
\end{equation}
where $\alpha=\max_{t} \left\{ \sup_{\bm{a},\bm{b} \in \mathcal{U}}\bm{a}^{\top}A_{t}^{-1}\bm{b} \right\}$ and $\beta \geq \Vert \bm{u}_{t,i}\Vert \Vert \bm{v}_{t} \Vert_{\star}$ with respect to some norm $\Vert \cdot \Vert$ and its dual norm $\Vert \cdot \Vert_{\star}$ for all $t\in [T] $ and $i\in [N]$. In particular, if $\mathcal{U}$ is a finite subset of $\lbrace 0,1 \rbrace^{d}$ we have further
\begin{equation}
\mathbb{E}[r_{n}] \leq O \left( d  \max_{t}\lbrace \lambda_{\min} (A_{t}) \rbrace\sqrt{\frac{ \ln N}{n}} \right),
\end{equation}
if $\beta \leq 1$.
\end{theorem}

\begin{proof}
We firstly bound the probability that the recommendation strategy does not select the optimal arm. 
Hence, we denote that when $J_{t}=j$ and $j \neq j^{\star}$ then 
\begin{equation}
\begin{split}
&\mathbb{P}_{\lbrace (U_{s},\bm{v}_{s}) \rbrace_{s=1}^{t}\sim \mathcal{D}^{t}}\left[ J_{t}= j  \right] \\
& \leq \mathbb{P}_{\lbrace (U_{s},\bm{v}_{s}) \rbrace_{s=1}^{t}\sim \mathcal{D}^{t}} \left[\left< \frac{ \sum_{s=1}^{t} \hat{\bm{v}}_{s}}{t} , \frac{\sum_{s=1}^{t}\bm{u}_{s,j}}{t}\right> - \left<\frac{\sum_{s=1}^{t}\hat{\bm{v}}_{s }}{t},\frac{\sum_{s=1}^{t}\bm{u}_{s,j^{\star}}}{t} \right>\geq 0     \right]\\
&=\mathbb{P}_{\lbrace (U_{s},\bm{v}_{s}) \rbrace_{s=1}^{t}\sim \mathcal{D}^{t}}\left[  \left< \frac{\sum_{s=1}^{t} \hat{\bm{v}}_{s}   }{t}, \frac{\sum_{s=1}^{t}\bm{u}_{s,j}-\bm{u}_{s,j^{\star}}}{t} \right> -\left<  \bm{v}, \bm{h}_{j} \right> \geq -\left<   \bm{v}, \bm{h}_{j} \right>    \right],
\end{split}
\end{equation}
where $\bm{h}_{j}=\bm{u}_{j}-\bm{u}_{j^{\star}},$ $\bm{h}_{s,j}=\bm{u}_{s,j}-\bm{u}_{s,j^{\star}}$.

Set that 
$$f(X_{1},\cdots, X_{t})=\left<\frac{\sum_{s=1}^{t} \hat{\bm{v}}_{s}   }{t}, \frac{\sum_{s=1}^{t}\bm{u}_{s,j}-\bm{u}_{s,j^{\star}}}{t} \right>_{H_{t}},$$
where $X_{s} =(U_{s}, \bm{v}_{s}) \sim \mathcal{D}$,
we have that $\mathbb{E}[f(X_{1},\cdots, X_{t})]=\langle \bm{v}, \bm{u}_{j}-\bm{u}_{j^{\star}} \rangle.$

On the other hand, $\forall s \in [t]$, we assume that $U_{s}$ and $\bm{v}_{s}$ are independent. So we have that
\begin{equation}
    \begin{split}
        &\mathbb{P}_{(U_{s},\bm{v}_{s}) \sim \mathcal{D}}\left[  \left<\frac{\sum_{s=1}^{t} \hat{\bm{v}}_{s}}{t},\frac{\sum_{s=1}^{t} \bm{h}_{s,j}  }{t}\right>- \langle \bm{v},\bm{h}_{j} \rangle \geq -\langle \bm{v},\bm{h}_{j}   \rangle                   \right]\\
        &= \int_{U_{s} \in \mathcal{U} ,\forall s\in [t]}\mathbb{P}_{\bm{v}_{s} \sim \mathcal{D}_{v}}\left[ \left. \left<\frac{\sum_{s=1}^{t} \hat{\bm{v}}_{s}}{t},\frac{\sum_{s=1}^{t} \bm{h}_{s,j}  }{t}\right>- \langle \bm{v},\bm{h}_{j} \rangle \geq -\langle \bm{v},\bm{h}_{j}   \rangle       \right| U_{1},\cdots, U_{t} \right] d\mathcal{D}_{U}\\
        & \leq \max_{U_{1},\cdots, U_{t}} \mathbb{P}_{\bm{v}_{s} \sim \mathcal{D}_{v}} \left[ ... \right| U_{1},\cdots, U_{t}]
    \end{split}
\end{equation}

Therefore, for a fixed sequence of $\lbrace U_{1},\cdots, U_{t} \rbrace$ we have that 
\begin{equation} \label{equ:incorrect}
\begin{split}
&\left|f(X_{1},\cdots,X_{s},\cdots, X_{t})-f(X_{1},\cdots, X_{s}^{'},\cdots, X_{t}) \right|\\
&\leq  \left|\frac{1}{t} \left< \frac{\sum_{s=1}^{t}\bm{h}_{s,j}}{t}, A^{-1}_{s} \bm{u}_{s,I_{s}} (\bm{u}_{s,I_{s}} (\bm{v}_{s} - \bm{v}_{s}^{\prime}))           \right> \right|\\
& \leq \frac{2}{t} \left(     \frac{\sum_{s=1}^{t}\bm{h}_{s,j}}{t}  \right)^{\top} A_{s}^{-1 } \bm{u}_{s,I_{s}} \beta \\
& \leq \frac{2}{t} \max_{\bm{a},\bm{b} \in \mathcal{U}} \bm{a}^{\top }A_{s}^{-1} \bm{b} \beta\leq 2\frac{\alpha}{t} \beta,
\end{split}
\end{equation}
where the first inequality is due to $\forall \bm{l}_{t} \in [0,1]^{N},$ $\Vert \bm{u}_{t,i} \Vert \Vert \bm{v}_{t} \Vert_{\star} \leq \beta $ and third inequality is from the the definition of $\alpha$ in the algorithm.

Next, due to the McDiarmid's inequality, we have that for all $j \neq j^{\star}$
\begin{equation*}
\begin{split}
\mathbb{P}_{U,\bm{v}}[J_{t} = j] 
&\leq \max_{U_{1},\cdots, U_{t}} \mathbb{P}_{\bm{v}}\left[ \left. \left< \frac{\sum_{s=1}^{t} \hat{\bm{v}}_{s}   }{t}, \frac{\sum_{s=1}^{t}\bm{h}_{s,j}}{t} \right> -\left<  \bm{v}, \bm{h}_{j} \right> \geq -\left<   \bm{v}, \bm{h}_{j} \right> \right| U_{1},\cdots, U_{t} \right] \\
& \leq \exp \left(   -\frac{2 \times \langle \bm{h}_{j}, \bm{v} \rangle^{2}       }{  \sum_{s=1}^{t} (\frac{\alpha \beta}{t})^{2}     }     \right) \\
& \leq \exp \left(   -\frac{2 \times \langle \bm{h}_{j}, \bm{v}  \rangle^{2} t }{(\alpha \beta)^{2}}             \right)
\end{split}
\end{equation*}

Furthermore, by defining $\Delta_{j}=(\bm{u}_{j^{\star}}-\bm{u}_{j})\cdot \bm{v}=\bm{l}(j^{\star})-\bm{l}(j)$ we have:
\begin{equation}
\begin{split}
\mathbb{E}\left[ r_{n} \right] 
&=\sum_{j=1}^{N} \Delta_{j}\mathbb{P}\left[ J_{t} = j \right] 
\leq \sum_{j: \Delta_{j} >\epsilon}\Delta_{j} \mathbb{P}\left[ J_{t}=j     \right] + \sum_{j : \Delta_{j} \leq \epsilon}\mathbb{P}\left[ J_{t} = j \right] \\
& \leq  \sum_{j: \Delta_{j} \geq \epsilon}\Delta_{j } \exp \left( -\frac{2 \Delta_{j}^{2}n    }{(\alpha \beta))^{2}}   \right)                   +\epsilon
\end{split}
\end{equation}

A function $f: x\rightarrow x\exp(-Cx^{2})$ is a decreasing function on $[\frac{1}{\sqrt{2C}},1]$ for any $C >0$. Taking $C=\frac{n}{2(\alpha \beta))^{2}}$ we obtain that
\begin{equation}
\mathbb{E}\left[ r_{n} \right] \leq \epsilon + (N-1)\epsilon \exp\left(-\frac{\epsilon^{2}n}{2(\alpha \beta))^{2}}\right),
\end{equation}
when $\epsilon \geq \frac{1}{\sqrt{\frac{n}{2(\alpha \beta)^{2} }}}$.
Particularly, we have that
\begin{equation}
\mathbb{E}\left[ r_{n} \right] \leq O\left(    \alpha \beta \sqrt{\frac{  \ln N }{ n }}         \right),
\end{equation}
by setting $\epsilon=\sqrt{\alpha \beta\frac{ 2\ln N  }{ n}}.$

In particular, if $\bm{u}_{t,i} \in \lbrace 0,1 \rbrace^{d}$ we have that
\begin{equation}
    \max_{\bm{a},\bm{b} \in \mathcal{U}} \bm{a}^{\top }A_{t}^{-1} \bm{b} \leq \Vert \bm{a} \Vert_{2} \Vert \bm{b} \Vert_{2} \lambda_{\min}(A_{t}) \leq  d \lambda_{\min}(A_{t})
\end{equation}
\end{proof}
\begin{remark}
 \cite{cesa2012combinatorial} states several applications for binary case of domain $\mathcal{U}$. For instance, $\lambda_{\min}(A_{t})$ can be lower bounded by $\frac{1}{4}$, if $\Vert \bm{v}_{t}\Vert_{2} \leq 1 $ and $\bm{u}_{t,i} \in \lbrace 0,1\rbrace^{d}$. It implies that our simple regret bound can achieve $O\left(d \sqrt{\frac{\ln N}{n}} \right)$. See details in section 5.2 in \citep{cesa2012combinatorial}.
\end{remark}

\subsection{Low rank reward with fixed kernel}
In this part, we consider the pure exploration in multi-armed bandits with low rank reward. The fixed kernel assumption can be seen as an easier case of time-varying kernel case. In this part, we provide an algorithm with Bary-centric spanner \cite{awerbuch2004adaptive,cesa2012combinatorial,dani2007price} to sample the arms to explore. Our proposed algorithm can achieve the regret bound $O\left( d\sqrt{\frac{\ln N}{n}} \right)$ without any further constraints over the kernel matrix and the seed vectors. 

In this sub-section, we denote the indices of the Bary-centric spanner as $\lbrace i_{1},\cdots, i_{d} \rbrace \subseteq [N].$ The exploration strategy selects $i_{j}$ if $t \mod j =0$ and the recommendation strategy is as follows:
\begin{equation} \label{algorithm:main}
J_{t}=\argmax_{j\in [N]} \bm{u}_{j} \cdot (V^{-1}\hat{\bm{l}}_{t})
\end{equation} where
\begin{equation}
\hat{\bm{l}}_{t}=\left(\frac{\sum_{s \in T_{i_{1}}(t)}\bm{l}_{s}(i_{1})}{|T_{i_{1}}(t)|}   ,\cdots,          \frac{\sum_{s\in T_{i_{d}}(t)}\bm{l}_{s}(i_{d})}{|T_{i_{d}}(t)|}                \right),
\end{equation}
and $V \in \mathbb{R}^{d\times d}$ is defined later.


\begin{theorem} \label{thm:maximal-determinant}
For a full rank matrix $U=\left[ \bm{u}_{1}^{\top},\cdots, \bm{u}_{N}^{\top}       \right]^{\top}$,
where $\bm{u}_{i}\in \mathbb{R}^{d}$ is the $i$-th row vector, we can find a full rank $d\times d$ matrix composed by $V=[\bm{u}_{i_{1}}^{\top},\cdots, \bm{u}_{i_{d}}^{\top}]^{\top}$ such that
\begin{equation}\label{algorithm:selecting-kernel}
\det (V)=\argmax_{[i_{1},\cdots i_{d}] \subset [N]} \left|\det \left( [\bm{u_{i_{1}}}^{\top} ,\cdots, \bm{u}_{i_{d}}^{\top}]^{\top} \right)\right|.
\end{equation}
Then for all $i$, we can find a $d$-dimensional vector $\bm{w}\in [-1,1]^{d}$ such that $\bm{u}_{i}^{\top}=V\bm{w}$.
\end{theorem}
This theorem indicates the Bary-centric spanner $\bm{u}_{i_{1}},\cdots, \bm{u}_{i_{d}}$ for the fixed kernel $\lbrace \bm{u}_{1},\cdots, \bm{u}_{N} \rbrace.$

See proof in supplementary material.

Now we give our main result.
\begin{theorem}\label{thm:regret-bound-fixed-kernel}
Running uniformly allocation to the selected $\lbrace i_{1},\cdots, i_{d}\rbrace \subseteq [N]$ arms in Equation (\ref{algorithm:selecting-kernel}) for the low rank structure and above mentioned recommendation strategy in Equation (\ref{algorithm:main}). We have an upper bound to the simple regret after $n$ times exploration:
\begin{equation}
\mathbb{E} \left[ r_{n} \right] \leq O \left(  \sqrt{\frac{ d^{2} \ln N }{ n }}  \right).
\end{equation}
\end{theorem}

\textbf{Proof sketch:}
The basic idea of proving Theorem \ref{thm:regret-bound-fixed-kernel} is similar to the previous part. Due to the Bary-centric spanner and the McDiarmid' s inequality, we can bound the probability of selecting the ``wrong" arm to recommend on the final round. 
Concretely, $$\mathbb{P}[J_{t} \neq j^{\star}] \leq \exp \left(  -\frac{  \langle \bm{u}_{j}-\bm{u}_{j^{\star}}, \bm{v} \rangle^{2} \lfloor \frac{t}{d} \rfloor }{2d}          \right).$$
Then we follow the same analysis to the upper bound of the simple regret and obtain our conclusion. See the detail of the proof in the Supplementary material.

\begin{remark}\label{remark:recover-original}
When $d=N$, by a more refined analysis, we can recover the simple regret in $O\left( \sqrt{\frac{N \ln N}{n}}\right).$ See more details in supplementary material.

\end{remark}

We give an efficient algorithm in the following part, which approximates the Bary-centric spanner in polynomial times.  \cite{awerbuch2004adaptive} provides an approximation algorithm for calculating the Bary-centric spanner effectively. With the time cost $O(d \log_{C}d)$ the algorithm can output a $d \times d$ matrix $\widehat{V}$ such that for any $\bm{u}_{i} \in U$  there exists $\bm{w} \in [-C,C]^{d}$ where $\bm{u}_{i}^{\top} =\widehat{V}\bm{w}$.

\begin{algorithm}
\caption{$C$-approximate algorithm for pure exploration low rank bandits}
\label{alg:C-approximation}
\begin{algorithmic}
\REQUIRE An oracle  $\mathcal{O}(U)=\widehat{V}$ as the $C$-approximation algorithm for Bary-centric spanner.
\STATE When $t=0$ calculate $\widehat{V}=\mathcal{O}(U)$ and denote $\widehat{V}=\left[ \bm{u}_{i_{1}}^{\top},\cdots, \bm{u}_{i_{d}}^{\top} \right]^{\top}$, where $\lbrace i_{1},\cdots, i_{d}\rbrace \subseteq [N]$.
\FOR {$t=1,\cdots, n$}
\STATE Select arm $i_{s}$ when $s= t \mod d $ and $s \neq 0$; $i_{d}$, otherwise.
\STATE Receive $\bm{l}_{t}(i_{s})$ and calculate
\begin{equation}
\hat{\bm{l}}_{t}=\left(     \frac{\sum_{s\in T_{i_{1}}(t)} \bm{l}_{s}(i_{1})}{|T_{i_{1}}(t)|},\cdots,    \frac{\sum_{s\in T_{i_{d}}(t)} \bm{l}_{s}(i_{d})}{|T_{i_{d}}(t)|}        \right).
\end{equation}
\IF{the environment sends a stopping signal}
\STATE Set $n$ to $t$.
\STATE Recommend $J_{n}=\max_{j \in [N]} \left< \bm{u}_{j},\hat{V}^{-1}\hat{\bm{l}}_{t}  \right>$ and stop.
\ENDIF
\ENDFOR
\end{algorithmic}
\end{algorithm}

\begin{proposition}
Running our algorithm for $n$ rounds, we have the upper bound to the simple regret as follows:
\begin{equation}
\mathbb{E}\left[ r_{n} \right] \leq O\left(  Cd \sqrt{\frac{  \ln N}{n}} \right)
\end{equation}
\end{proposition}

\subsection{Lower bound}
In this part, we provide a lower bound to pure exploration with low rank sequence.  Generally, there is a gap $\sqrt{\ln N}$ to our upper bound to the pure exploration in separated setting when $d \leq \ln N.$ It implies that our proposed algorithm is sub-optimal. Meanwhile our proposed lower bound is similar to the lower bound of linear bandits but with a $\sqrt{n}$ factor \cite{dani2007price}. In comparison to the united setting, we also provide a lower bound when $d \geq \ln N$, which states that there is an optimal result for the united setting.

To demonstrate our main result, firstly we extend the lower bound for pure exploration with multi-armed bandits from the traditional result in \cite{auer2002nonstochastic}. Then, in our proof, we are going to reduce the problem for low rank sequence to a corresponding 2-armed bandits, and obtain the lower bound to pure exploration with low rank structure.

\begin{lemma}\label{lem:lower-bound-multi-arm-pure-exploration}
The lower bound to the simple regret for the stochastic multi-armed bandits is $\Omega\left(\sqrt{\frac{N}{n}}\right)$.
\end{lemma}

See the proof in Appendix 

Now we attempt to show the lower bound to low rank loss.
\begin{theorem} \label{thm:lower-bound}
  The lower bound to the simple regret for the multi-armed bandits with low rank sequence is $\Omega \left( \frac{d}{\sqrt{n}}\right),$ where the kernel matrix is a fixed $N \times d$ matrix.  
\end{theorem}

\begin{proof}
Without loss the generality, we assume that there are $N=2^{d}$ arms that each arm $\bm{u}_{i} \in \mathcal{A}=\lbrace +1, -1 \rbrace^{d}$ and seed vector $\bm{v}_{t} \in \mathcal{L}= \lbrace \pm\bm{e}_{i}: 1\leq i \leq d\rbrace$. We give a distribution $\mathcal{P}_{\bm{b}}$ upon $\mathcal{L}$ such that
\begin{equation}
\mathcal{P}_{\bm{b}}(\bm{x})=
\begin{cases}
\frac{1+\bm{b}(i)\epsilon}{2d} & \text{$\bm{x}=\bm{e}_{i}$}\\
\frac{1-\bm{b}(i)\epsilon}{2d} & \text{$\bm{x}=-\bm{e}_{i}$},
\end{cases}
\end{equation}
for some $\epsilon \in (0,1/2)$ and $\bm{b} \in \lbrace  \pm 1\rbrace^{d}$.

Hence, the simple regret is formulated as follows:
\begin{equation}
\begin{split}
r_{n}=\max_{i \in [N]}\mathbb{E}_{\bm{v}_{n} \sim \mathcal{P}_{\bm{b}}} [\bm{u}_{i} \cdot \bm{v}_{n}]-\mathbb{E}_{\bm{v}_{n} \sim \mathcal{P}_{\bm{b}}}[\bm{u}_{\varphi_{n}} \cdot \bm{v}_{n}]
=\max_{i\in [N]}\bm{u}_{i}\bm{v}-\bm{u}_{\varphi_{n}}\bm{v},
\end{split}
\end{equation}
where $\bm{v}=\mathbb{E}[\bm{v}_{n}]=\sum_{i=1}^{d}\frac{1+\bm{b}(i)\epsilon}{2d}\cdot \bm{e}_{i}+\frac{1-\bm{b}(i)\epsilon}{2d}\cdot (-\bm{e}_{i}).$
Hence, we have that
\begin{equation}
\bm{b}=\argmax_{\bm{u}\in \mathcal{A}}\bm{u} \cdot \bm{v},
\end{equation}
since $\bm{v}=\mathbb{E}[\bm{v}_{n}]=\left(\frac{\bm{b}(1)\epsilon}{2d}, \cdots, \frac{\bm{b}(d)\epsilon}{d}\right).$

We simplify that $\bm{u}_{\varphi_{n}}=\bm{u}_{n} \in \lbrace \pm 1\rbrace^{d}$, then we have that
\begin{equation}
\begin{split}
r_{n}&=\bm{b}\cdot \bm{v}-\bm{u}_{n} \cdot \bm{v}
=\sum_{i=1}^{d} \bm{v}(i)(\bm{b}(i)-\bm{u}_{n}(i))\\
&=\sum_{i=1}^{d} \frac{\bm{b}(i)\epsilon}{d} \left( \bm{b}(i)-\bm{u}_{n}(i)\right)
=\sum_{i=1}^{d} \frac{\bm{b}(i)\epsilon}{d}2 \mathbb{I}_{[\bm{b}(i) \neq \bm{u}_{n}(i)]}.
\end{split}
\end{equation}

Next, we define $r_{n}^{i}(\bm{b}(i))=\frac{\bm{b}(i)\epsilon}{d}2 \mathbb{I}_{[\bm{b}(i) \neq \bm{u}_{n}(i)]}.$ Then, we have that $r_{n}^{i}(\bm{b}(i))$ is at least the regret that would be incurred if the algorithm knew that the adversary was using one of the $\mathbb{P}_{\bm{b}}$ distributions and also knew $\lbrace  \bm{b}(j): j \neq i \rbrace$. In this case it would know that the algorithm is incurred a $(\pm 1)$ Bernoulli distribution of $\bm{u}_{n} \cdot \bm{v}_{n}$ with mean:
\begin{equation}
\mathbb{E} \left[\sum_{j \neq i}\bm{v}_{n}(j) \bm{u}_{n}(j)    \right]+\bm{b}(i)\bm{u}_{n}(i)\frac{\epsilon}{d}.
\end{equation}

Clearly, since the algorithm knew $\lbrace  \bm{b}(j): j \neq i \rbrace$ the algorithm is supposed to set $\bm{u}_{n}(j)=\bm{b}(j),$ for any $j \neq i$.
Therefore we have that
$$\mathbb{E} \left[\sum_{j \neq i}\bm{v}_{n}(j) \bm{u}_{n}(j)    \right]=\sum_{j \neq i}1\times \frac{1+\bm{b}(j)\epsilon}{2d} \cdot \bm{b}(j)+ (-1) \times \frac{1-\bm{b}(j)\epsilon}{2d}\cdot \bm{b}(j)=\frac{(d-1)\epsilon}{d}$$

On the other hand, we can consider this game as a 2-armed bandit problem with
\begin{itemize}
\item arm 1: Selecting $\bm{u}_{n}(i)=1$ with mean $\frac{(d-1)\epsilon + \bm{b}(i)\epsilon}{d};$
\item arm 2: Selecting $\bm{u}_{n}(i)=-1$ with mean $\frac{(d-1)\epsilon - \bm{b}(i)\epsilon}{d}.$
\end{itemize}

Utilising the Lemma \ref{lem:lower-bound-multi-arm-pure-exploration}, we have that
$$\mathbb{E}[r_{n}^{i}(\bm{b}(i))] \geq \Omega \left( \sqrt{\frac{1}{n}} \right).$$ Therefore,

\begin{equation}
\mathbb{E}[r_{n}]=\sum_{i=1}^{d} \mathbb{E}\left[ r_{n}^{i}(\bm{b}(i))  \right]  \geq \Omega\left(\frac{d}{\sqrt{n}}\right).
\end{equation}

\end{proof}

In the following corollary, we show another form of the lower bound $\Omega\left( \sqrt{\frac{d \ln N}{n}} \right),$ when $d \geq \ln N.$

\begin{corollary}
For all $N=2^{k}$, $ k \in \lbrace 1,\cdots, d \rbrace$ and all outputs of the recommendation strategy, there is a strategy of setting seed vector $\bm{v}_{t}$ such that
\begin{equation*}
    \mathbb{E}\left[ r_{n}  \right] \geq \Omega\left(  \sqrt{\frac{d \ln N}{n}} \right).
\end{equation*}
\end{corollary}

\begin{proof}
We break the the $d$-dimensional feature space into smaller $k$ dimensional space. Setting $\beta=\frac{d}{k}$, for each $k$-dimensional space, with $N$-arms, round $\frac{n}{\beta}$ and the expectation of the  seed vector $\bm{v}_{k} \in \mathbb{R}^{k}$ we have that
$$\mathbb{E}[r^{j}_{\frac{n}{\beta}}]\geq \Omega\left( k \sqrt{\frac{1}{\frac{n}{\beta}}}\right)=\Omega\left( \sqrt{\frac{dk }{n}} \right)$$ according to Theorem \ref{thm:lower-bound}.

At round $\tau=(j-1)\frac{n}{\beta}+t$ the seed vector is $\bm{v}_{t,i}$ such that $(0, \cdots, 0, \bm{v}_{t,i},0 \cdots,0) \in \mathbb{R}^{d},$ and similar to the reward sequence. The strategy of setting seed vector is cycling in each round block with size $\frac{n}{\beta}$. Therefore we have that
\begin{equation}
    \mathbb{E}\left[ r_{n}\right]=\frac{1}{\beta} \sum_{j=1}^{\beta}\mathbb{E}\left[  r^{j}_{\frac{n}{\beta}}\right]=\Omega\left(\sqrt{\frac{dk}{n}}\right)=\Omega\left( \sqrt{\frac{d \ln N}{n}}\right).
\end{equation}
\end{proof}

\begin{remark}
The lower bounds we offered in this section are for both cases of the separated (oblivious sampling) and united (non-oblivious sampling) exploration and recommendation strategies.
In the case of united setting, there is an upper bound to the simple regret by $O\left( \sqrt{\frac{d \ln N}{n}}\right)$. This bound can be obtained, when the exploration strategy is UCB and the recommendation strategy is to output the average distribution of the outputs of UCB. 
In the oblivious sampler case, there is still a $\sqrt{\ln N}$ gap between the upper bound and the lower bound for both fixed and time-varying kernel. 

\end{remark}

%% file: conclusion.tex
\section{Conclusion and future work}
In this paper we extend the pure exploration for multi-armed bandits to the low rank reward. In our problem setting, we restrict that the exploration and recommendation strategy only share few information. It implies that the exploration strategy is requested to sample the arms obliviously. Meanwhile, we involve the kernel space of the reward sequence. Furthermore, we consider the kernel is fixed and stochastic respectively and propose the efficient algorithms and demonstrate simple regret bounds. We provide  algorithms which guarantee  regret bound $O\left( d\sqrt{\frac{\ln N}{n}}  \right)$ for both fixed and stochastic kernel setting.

Clearly, our designed algorithms require the prior information of the kernel in each round. A challenging future work is to design algorithms without the kernel information. Meanwhile the adversarial Multi-armed bandits with low rank reward is also an open problem left.

%% file: appendix.tex
\section{Supplementary material}

\subsection{Proof of Theorem \ref{thm:maximal-determinant}}

\begin{proof}
We prove our result by contradiction. Suppose that $\Vert \bm{w}\Vert_{\infty }> 1$, then we  assume that there exists one $i \in [N]$ such that $\bm{u}_{i}=\sum_{j=1}^{d} \bm{w}(j)\bm{u}_{i_{j}}$ and
without loss the generality we can assume that $\bm{w}(1) >1$. Then we have that
\begin{equation*}
\begin{split}
&\left| \det\left( [\bm{u}_{i}, \bm{u}_{i_{2}},\cdots, \bm{u}_{i_{d}}    ]        \right)     \right| \\
&= \left| \det\left( \left[\sum_{j=1}^{d}\bm{w}(j)\bm{u}_{i_{j}}, \bm{u}_{i_{2}},\cdots, \bm{u}_{i_{d}}    \right]        \right)     \right| \\
&= \left| \det\left( \sum_{j=1}^{d} \bm{w}(j) \left[\bm{u}_{i_{j}}, \bm{u}_{i_{2}},\cdots, \bm{u}_{i_{d}}    \right]        \right)     \right| \\
&= |\bm{w}(1)| \left| \det\left( [\bm{u}_{1}, \bm{u}_{i_{2}},\cdots, \bm{u}_{i_{d}}    ]        \right)     \right|\\
&= |\bm{w}(1)| \det (V).\\
\end{split}
\end{equation*}

It leads to a contradiction that there is a matrix $V^{\prime}=[\bm{u}_{i}^{\top},\bm{u}_{i_{2}}^{\top},\cdots, \bm{u}_{i_{d}}^{\top}]^{\top}$ with $\det (V^{\prime}) > \det (V).$
Therefore, we have our conclusion.
\end{proof}

\subsection{Proof of Theorem \ref{thm:regret-bound-fixed-kernel}}
\begin{proof}
\begin{equation}
\begin{split}
&\mathbb{P} \left[ J_{t} \neq j^{\star}             \right] \\
&=\mathbb{P} \left[\langle \bm{u}_{j}-\bm{u}_{j^{*}} , V^{-1} \cdot \hat{\bm{l}}_{t} \rangle \geq 0           \right]\\
&=\mathbb{P} \left[  \langle \bm{u}_{j}-\bm{u}_{j^{*}} , V^{-1} \cdot \hat{\bm{l}}_{t} \rangle -\langle \bm{u}_{j}-\bm{u}_{j^{*}},\bm{v} \rangle \geq -\langle \bm{u}_{j}-\bm{u}_{j^{*}},\bm{v} \rangle      \right]
\end{split}
\end{equation}

Denoting $\bm{h}_{j}=\bm{u}_{j}-\bm{u}_{j^{*}}$, we can further rewrite the above equation as follows:
\begin{equation}
\begin{split}
&\mathbb{P} \left[ \hat{\mu}_{j,t}-\hat{\mu}_{j^{*},t} \geq 0               \right]\\
&=\mathbb{P} \left[ \langle \bm{h}_{j} \cdot V^{-1}, \hat{\bm{l}}_{t} \rangle -\langle \bm{h}_{j} V^{-1}, \tilde{\bm{l}} \rangle \geq -\langle \bm{h}_{j},\bm{v} \rangle \right]
\end{split}
\end{equation}

Setting $X_{s}=\bm{l}_{s}(i_{s \mod d})$ when $(s \mod d)\neq 0,$ $X_{s}=\bm{l}_{s}(i_{d})$ otherwise, and
\begin{equation}
f(X_{1},\cdots, X_{t})=\left< \left( \frac{\sum_{s\in T_{i_{1}}(t)} X_{s}}{|T_{i_{1}}(t)|},\cdots,    \frac{\sum_{s\in T_{i_{d}}(t)} X_{s}}{|T_{i_{d}}(t)|}      \right), V^{-1}\bm{h}_{j} \right>,
\end{equation}
then we have that
\begin{equation*}
\begin{split}
\left|f(X_{1},\cdots,X_{s},\cdots, X_{t})-f(X_{1},\cdots, X_{s}^{'},\cdots, X_{t}) \right| &\leq  \left| \left<  \frac{X_{s}-X_{s}^{'}}{ |T_{i_{ s\mod d}}(t)|   }, V^{-1}(\bm{u}_{j^{*}}-\bm{u}_{j})    \right>    \right|  \\
&\leq \left\lVert \frac{X_{i}-X_{i}^{'}}{ |T_{i_{ s\mod d}}(t)|  } \right\lVert_{1}\Vert V^{-1}(\bm{u}_{j^{*}}-\bm{u}_{j}) \Vert_{\infty} \\
&\leq \frac{1}{\lfloor \frac{t}{d} \rfloor} \cdot 2,
\end{split}
\end{equation*}
while the coefficients of $\bm{u}_{j^{*}},\bm{u}_{j}$ are constrained in $[-1,1]$ with respect to $V$.
Thus, we set that $c_{i}=\frac{2}{\lfloor \frac{t}{d} \rfloor}$.

Utilising McDiarmid's inequality to Equation above we obtain that
\begin{equation}
\begin{split}
&\mathbb{P} \left[ \langle \bm{h}_{j} \cdot V^{-1}, \hat{\bm{l}}_{t} \rangle -\langle \bm{h}_{j} V^{-1}, \tilde{\bm{l}} \rangle \geq -\langle \bm{h}_{j},\bm{v} \rangle \right] \\
& \leq \exp \left( -\frac{2 \times \langle \bm{h}_{j}, \bm{v}\rangle^{2} }{\sum_{i=1}^{t} \left(   \frac{2}{\lfloor \frac{t}{d}\rfloor }     \right)^{2}}                 \right) 
 \leq \exp \left( -\frac{1 \times \langle \bm{h}_{j}, \bm{v}\rangle^{2} \lfloor \frac{t}{d}\rfloor }{2   d       }                 \right),
\end{split}
\end{equation}
where $n$ is the total rounds for exploration.

Next we give the upper bound to the simple regret.

\begin{equation}
\begin{split}
\mathbb{E}\left[ r_{n} \right] 
&=\sum_{j=1}^{N} \Delta_{j}\mathbb{P}\left[ I_{t} = j \right] 
\leq \sum_{j: \Delta_{j >\epsilon}}\Delta_{j} \mathbb{P}\left[ \hat{\mu}_{j,t} \geq \hat{\mu}_{j^{*},t}     \right] + \sum_{j : \Delta_{j} \leq \epsilon}\mathbb{P}\left[ I_{t} = j \right] \\
& \leq  \sum_{j: \Delta_{j} \geq \epsilon}\Delta_{j }\exp \left( -\frac{\Delta_{j}^{2}n}{2d}                   \right)                        +\epsilon
\end{split}
\end{equation}

A function $f: x\rightarrow x\exp(-Cx^{2})$ is a decreasing function on $[\frac{1}{\sqrt{2C}},1]$ for any $C >0$. Taking $C=\frac{n}{2d^{2}}$ we obtain that
\begin{equation}
\mathbb{E}\left[ r_{n} \right] \leq \epsilon + (N-1)\epsilon \exp\left(-\frac{\epsilon^{2}n}{2d^{2}}\right),
\end{equation}
when $\epsilon \geq \frac{1}{\sqrt{\frac{n}{2d^{2} }}}$.
Particularly, we have that
\begin{equation}
\mathbb{E}\left[ r_{n} \right] \leq O\left(     \sqrt{\frac{ d^{2} \ln N }{ n }}         \right),
\end{equation}
by setting $\epsilon=\sqrt{\frac{ 2d^{2}\ln N  }{ n}}.$
\end{proof}

\subsection{Details of Remark \ref{remark:recover-original}}

Let us have a review of the case of $d=N$, and offer a more refined analysis. Due to the fact that $\bm{u}_{i}\cdot V^{-1}\hat{\bm{l}}_{t}=\hat{\bm{l}}_{t}(i), \forall i\in[N],$ when $V=U$, we have further:
\begin{equation}
\begin{split}
f(X_{1},\cdots, X_{t})&=\left< \left( \frac{\sum_{s\in T_{i_{1}}(t)} X_{s}}{|T_{i_{1}}(t)|},\cdots,    \frac{\sum_{s\in T_{i_{N}}(t)} X_{s}}{|T_{i_{N}}(t)|}      \right), V^{-1}\bm{h}_{j} \right> \\
&=\left< \bm{u}_{j}, V^{-1} \hat{\bm{l}}_{t}           \right>-\left<  \bm{u}_{j}, V^{-1} \hat{\bm{l}}_{t}           \right>\\
&=\frac{\sum_{s\in T_{j}(t)} X_{s}}{|T_{j}(t)|}-\frac{\sum_{s\in T_{j^{*}}(t)} X_{s}}{|T_{j^{*}}(t)|}.
\end{split}
\end{equation}
Next, we can have that
\begin{equation}
\left|f(X_{1},\cdots,X_{s},\cdots, X_{t})-f(X_{1},\cdots, X_{s}^{'},\cdots, X_{t}) \right|=\left| \frac{\sum_{s\in T_{j}(t)} X_{s}}{|T_{j}(t)|}-\frac{\sum_{s\in T_{j^{*}}(t)} X_{s}}{|T_{j^{*}}(t)|} \right|,
\end{equation}
so
\begin{equation}
\left|f(X_{1},\cdots,X_{s},\cdots, X_{t})-f(X_{1},\cdots, X_{s}^{'},\cdots, X_{t}) \right|
 \leq \begin{cases}
\frac{2}{\lfloor \frac{t}{N} \rfloor} & \text{$s \in T_{j}(t) \cup T_{j^{*}}(t)$}\\
0 & \text{otherwise}
\end{cases}.
\end{equation}
Now we can set $c_{i}=\begin{cases}
\frac{2}{\lfloor \frac{t}{N} \rfloor} & \text{$s \in T_{j}(t) \cup T_{j^{*}}(t)$}\\
0 & \text{otherwise}
\end{cases}.$

Therefore, utilising the McDiarmid's inequality  we obtain that
\begin{equation}
\begin{split}
\mathbb{P} \left[ \hat{\mu}_{j,t}-\hat{\mu}_{j^{*},t} \geq 0\right] & \leq \exp \left( -\frac{2 \times \langle \bm{h}_{j}, \bm{v}\rangle^{2} }{\sum_{i=1}^{t} \left(  c_{i} \right)^{2}}\right)\\
& \leq \exp \left( -\frac{1 \times \langle \bm{h}_{j}, \bm{v}\rangle^{2} \lfloor \frac{t}{N}\rfloor }{2         }                 \right),
\end{split}
\end{equation}
and the previous bound $O\left(\sqrt{\frac{N \ln N}{n}} \right).$

\subsection{Proof of Lemma \ref{lem:lower-bound-multi-arm-pure-exploration}}
To prove Lemma \ref{lem:lower-bound-multi-arm-pure-exploration}, we need to assume the following two allocation strategies of the environment.
\begin{itemize}
    \item We assume that arm $i^{*}$ is  Bernoulli distributed with parameter $\frac{1+\epsilon}{2}$ and others with parameter $\frac{1-\epsilon}{2}$. The arm $i^{*}$ is selected by the environment at the beginning of the learning process uniformly.         We denote that $\mathbb{P}_{i^{*}}$ to the probability conditioned on $i^{*}$ being the ``good'' arm.  
    \item We assume that all arms are Bernoulli distributed with parameter $\frac{1-\epsilon}{2}$, and we denote this allocation by $\mathbb{P}_{0}$.
\end{itemize}
To simplify the notations we denote $N_{i}$ to the number of arm $i$ chosen by the exploration strategy $\mathcal{A}$. Let $e_{t}=\bm{l}_{t}(i_{t})$ be a random variable denoting the reward of exploration at time $t$, and let $\bm{e}_{t}$ denote the sequence of rewards explored up through round $t: \bm{e}_{t}=(e_{1},\cdots, e_{t})$. For shorthand $\bm{e}_{T}=\bm{e}$.

Obviously any randomised recommendation strategy is equivalent to an a-priori random choice from the set of all deterministic strategies. 
On the one hand, the exploration processes the exploration arm $I_{t}$ according to the history $\bm{e}_{t-1}$, on the other hand, the recommendation strategy makes recommendation $J_{t}$ according to the reward sequence $\bm{e}_{t}$ at the end of round $t$.
Therefore, we can formally regard the algorithm $\mathcal{A}$ as a fixed function, which at round $t$, maps the exploration $\bm{e}_{t-1}$ to the next exploration $e_{t}$ and at the end of the round $t$ further to the recommendation $J_{t}$.

\begin{lemma}[\cite{auer2002nonstochastic}]
    Let $f:\lbrace 0,1\rbrace^{T} \rightarrow [0,M]$ be any bounded function defined on exploration sequence $\bm{e}$. Then, for any action $i^{\star}$
    \begin{equation}
        \mathbb{E}_{i^{*}}[f(\bm{e})] \leq \mathbb{E}_{0}[f(\bm{e})]+\frac{M}{2}\sqrt{\mathbb{E}_{0}[N_{i^{*}}]\ln (1-4\epsilon^{2})}.
    \end{equation}
\end{lemma}
\begin{proof}
        \begin{equation*}
        \begin{split}
            \mathbb{E}_{i^{*}}[f(\bm{e})]-\mathbb{E}_{0}[f(\bm{e})]&=\sum_{\bm{e}}f(\bm{e})(\mathbb{P}_{i^{*}}\lbrace 
 \bm{e}\rbrace-\mathbb{P}_{0}\lbrace\bm{e}\rbrace)\\
            &\leq \sum_{\bm{e}: \mathbb{P}_{i^{*}} \lbrace \bm{e} \rbrace \geq \mathbb{P}_{0}\lbrace \bm{e} \rbrace}f(\bm{e})(\mathbb{P}_{i^{*}}\lbrace 
 \bm{e}\rbrace-\mathbb{P}_{0}\lbrace\bm{e}\rbrace)\\
 & \leq M \sum_{\bm{e}: \mathbb{P}_{i^{*}} \lbrace \bm{e} \rbrace \geq \mathbb{P}_{0}\lbrace \bm{e} \rbrace}(\mathbb{P}_{i^{*}}\lbrace 
 \bm{e}\rbrace-\mathbb{P}_{0}\lbrace\bm{e}\rbrace)\\
 &=\frac{M}{2}\Vert \mathbb{P}_{i^{*}}-\mathbb{P}_{0} \Vert_{1}.
        \end{split}
    \end{equation*}
By $$\Vert \mathbb{P}_{i^{*}}-\mathbb{P}_{0}\Vert_{1}^{2} \leq 2\ln 2 \mathrm{KL}( \mathbb{P}_{0} \Vert \mathbb{P}_{i^{*}}),$$
we have that
\begin{equation*}
    \begin{split}
\mathrm{KL}( \mathbb{P}_{0} \Vert \mathbb{P}_{i^{*}})&= \sum_{t=1}^{T}\mathrm{KL} (  \mathbb{P}_{0}\lbrace e_{t}|\bm{e}_{t-1} \Vert \mathbb{P}_{i^{*}} \lbrace e_{t} | \bm{e}_{t-1} \rbrace   )\\
&=\sum_{t=1}^{T} \sum_{\bm{e}_{t} \in \lbrace 0,1 \rbrace^{t}} \mathbb{P}_{0} \lbrace \bm{e}_{t } \rbrace \ln \frac{ \mathbb{P}_{0} \lbrace e_{t} | \bm{e}_{t-1} \rbrace  }{ \mathbb{P}_{i^{*}} \lbrace e_{t }  | \bm{e}_{t-1} \rbrace}\\
&=\sum_{t=1}^{T } \sum_{\bm{e}_{t} \in \lbrace 0,1 \rbrace^{t}} \mathbb{P}_{0}\lbrace\bm{e}_{t-1}\rbrace\mathbb{P}_{0} \lbrace e_{t}| \bm{e}_{t-1} \rbrace \ln \frac{ \mathbb{P}_{0} \lbrace e_{t}=1 | \bm{e}_{t-1} \rbrace }{ \mathbb{P}_{i^{*}} \lbrace e_{t}=1| \bm{e}_{t-1} \rbrace } \\
&=\sum_{t=1}^{T}\mathbb{P}_{0}\lbrace i_{t} \neq i^{*} \rbrace \mathrm{KL}\left(\frac{1-\epsilon}{2} \left\lVert  \frac{1-\epsilon}{2}\right. \right)+\mathbb{P}_{0}\lbrace i_{t}=i^{*}\rbrace \mathrm{KL}\left(\frac{1-\epsilon}{2} \left\lVert \frac{1+\epsilon}{2} \right.\right)\\
&=\sum_{t=1}^{T} \mathbb{P}_{0}\lbrace i_{t} =i^{*} \rbrace \left( -\frac{1}{2} \ln (1-4\epsilon^{2}) \right) \\
&=\mathbb{E}_{0} [N_{i^{*}}] \left( -\frac{1}{2} \ln (1-4\epsilon^{2}) \right)
     \end{split}
\end{equation*}
The second equality is from the definition of the conditional relative entropy. The fourth equality is due to the fact that the conditional probability distribution $\mathbb{P}_{0}\lbrace e_{t}| \bm{e}_{t-1} \rbrace$ for $e_{t}$ is uniform upon $\lbrace 0,1 \rbrace.$ Meanwhile the conditional distribution $\mathbb{P}_{i^{*}} \lbrace e_{t} | \bm{e}_{t-1} \rbrace$ to $e_{t}$ is determined by the choice of $i^{*}$. Once $i^{*}=i_{t}$ the probability over  1 is $\frac{1+\epsilon}{2}$, otherwise, $\frac{1-\epsilon}{2}. $
\end{proof}

According to above Lemma, we may have the lower bound to the pure exploration problem in the following theorem.

\begin{theorem}
    For any recommendation strategy $\mathcal{A}$ and for the distribution on rewards given above, the expected regret of $\mathcal{A}$ is lower bounded by $\Omega\left( \sqrt{\frac{N}{n}}\right)$.
\end{theorem}

\begin{proof}
    We see the recommendation strategy $\mathcal{A}$ into two parts, in the first $n$ rounds, the strategy actually processes    exploration upon $N$ arms. At the end of round $n$, the strategy recommends an arm $J_{n}$ as the output of the recommendation.

We denote that $r_{n}$ to be the reward of the recommendation. Then, we have that
\begin{equation}
    \mathbb{E}_{i^{*}}[r_{n}]=\left(\frac{1+\epsilon}{2}  \cdot \mathbb{P}_{i^{*}}(J_{n}=i^{*})\right)+\left( \frac{1-\epsilon}{2} \cdot \mathbb{P}_{i^{*}}(J_{n} \neq i^{*}) \right)=\frac{1}{2}-\frac{\epsilon}{2}+\epsilon\mathbb{P}_{i^{*}}(J_{n}=i^{*}).
\end{equation}

Further,
\begin{equation}
    \mathbb{E}_{i^{*}}[\bm{l}(i^{*})-\bm{l}(J_{n})]=\epsilon-\epsilon\mathbb{P}_{i^{*}}(i_{n}=i^{*})
\end{equation}

Note that $\mathbb{P}_{i^{*}}(J_{n}=i^{*})=\mathbb{E}_{i^{*}}[\mathbb{I}_{J_{n}=i^{*}}].$

Then, we have that
\begin{equation}
    \frac{1}{N}\sum_{i^{*}=1}^{N}\mathbb{E}_{i^{*}}[\bm{l}(i^{*})-\bm{l}(J_{n})] =\epsilon \left(  1-\frac{1}{N} \sum_{i^{*}=1}^{N}\mathbb{P}_{i^{*}}(J_{n}=i^{*}) \right)
\end{equation}

Now, we set that $f(\bm{e}_{n})=\mathbb{I}_{J_{n}=i^{*}}$. Obviously $\mathbb{P}_{0}[f(\bm{e})]=1,$ while all expectations are the same.

Due to above lemma, we have that
\begin{equation}
    \frac{1}{N}\sum_{i^{*}=1}^{N}\mathbb{E}_{i^{*}}[\bm{l}(i^{*})-\bm{l}(J_{n})]\geq \epsilon
    \left( 1-\frac{1}{N}+\sqrt{\frac{T}{N} \ln (1-4\epsilon^{2})}\right),
    \end{equation}
where it follows from $\sum_{i=1}^{N}\sqrt{\mathbb{E}_{0}[N_{i}]} \leq \sqrt{Nn}$.

Then, we have that $\frac{1}{N}\sum_{i^{*}=1}^{T}\mathbb{E}_{i^{*}}[\bm{l}(i^{*})-\bm{l}(J_{n})] \geq \Omega\left( \frac{\epsilon}{2}-C\epsilon\sqrt{\frac{n}{N}} \right)$ by $-\ln (1-x) \leq 4\ln(4/3)x$ for $x\in [0,1/4].$

By setting $\epsilon=\frac{N}{n}$, we have our conclusion.
\end{proof}